\def\year{2021}\relax
\documentclass[letterpaper]{article} 
\usepackage{aaai21}  
\usepackage{times}  
\usepackage{helvet} 
\usepackage{courier}  
\usepackage[hyphens]{url}  
\usepackage{graphicx} 
\usepackage{natbib}  
\usepackage{caption} 
\title{Overestimation learning with guarantees}

\author{

 Adrien Gauffriau\textsuperscript{\rm 1}\thanks{seconded from Airbus Operations, Toulouse}
 Fran\c{c}ois Malgouyres\textsuperscript{\rm 2}\thanks{This work has benefited from the AI Interdisciplinary Institute ANITI. ANITI is funded by the French "Investing for the Future – PIA3" program under the Grant agreement n°ANR-19-PI3A-0004.},
 M\'elanie Ducoffe\textsuperscript{\rm 1}\thanks{seconded from Airbus AI Research, Toulouse}
 \\
 }

\affiliations{
    \textsuperscript{\rm 1} IRT Saint Exup\'ery, Toulouse, France \\
    \textsuperscript{\rm 2}Institut de Mathématiques de Toulouse ; UMR5219 Université de Toulouse ; CNRS ;\\ UPS IMT F-31062 Toulouse Cedex 9, France \\
    \{adrien.gauffriau, melanie.ducoffe\}@irt-saintexupery.com, Francois.Malgouyres@math.univ-toulouse.fr
}

\usepackage{amsmath}
\usepackage{amsthm}

\usepackage{algorithm}
\usepackage{algorithmic}

\usepackage{amssymb}
\usepackage{subcaption}
\usepackage{stmaryrd}

\usepackage{pict2e}
\linethickness{0.2mm}

\newcommand {\RR} {\mathbb R}

\newcommand {\NN} {\mathbb N}

\newcommand {\calD} {\mathcal{D}}

\newcommand {\wbf} { \mathbf{w}}      

\DeclareMathOperator {\argmin} {argmin}

\newtheorem{prop}{Proposition}
\newtheorem{thm}{Theorem}

\usepackage[squaren, Gray, cdot]{SIunits}

\newcommand {\bbf} { \mathbf{b}}   
 
\newcommand {\calR} {\mathcal R} 
\newcommand {\calC} {\mathcal C} 

\usepackage{color}

\usepackage{graphicx}  
\newtheorem{definition}{Definition}
\newcommand{\dominate}{Majoring Points}
\newcommand{\Gdominate}{Grid Based Majoring Points}
\newcommand{\Fdominate}{Function Adapted Majoring Points}
\newcommand{\Ddominate}{Data Adapted Majoring Points}
\newcommand{\DA}{Dichotomy Algorithm}
\newcommand{\toyone}{$1D$ synthetic experiment}

\newcommand{\baseline} {{\em baseline}}

\begin{document}
\maketitle
\begin{abstract}

We describe a complete method that learns a neural network which is guaranteed to over-estimate a reference function on a given domain. The neural network can then be used as a surrogate for the reference function.

The method involves two steps. In the first step, we construct an adaptive set of \dominate. In the second step, we optimize a well-chosen neural network to over-estimate the \dominate.

In order to extend the guarantee on the \dominate~to the whole domain, we necessarily have to make an assumption on the reference function. In this study, we assume that the reference function is monotonic.

We provide experiments on synthetic and real problems. The experiments show that the density of the \dominate~concentrate where the reference function varies. The learned over-estimations are both guaranteed to overestimate the reference function and are proven empirically to provide good approximations of it.

Experiments on real data show that the method makes it possible to use the surrogate function in embedded systems for which an underestimation is critical; when computing the reference function requires too many resources.

\end{abstract}

\section{Introduction}

\subsection{Overestimation guarantee}
In this paper, we consider a real value finite function $f$ defined on a compact domain $\calD$ and describe a method that finds optimal weights $\wbf^*$ and bias $\bbf^*$ such that the neural network $f_{\wbf^*,\bbf^*}$ both provides a good approximation of $f$:
\[f_{\wbf^*,\bbf^*} \sim f,
\] 
and is guaranteed to overestimate $f$ over $\calD$:
\begin{equation}\label{over-estim}
f_{\wbf^*,\bbf^*}(x) \geq f(x)\qquad\mbox{ for all }x\in\calD.
\end{equation}
Inspired by critical systems applications, we require \eqref{over-estim} to be guaranteed by a formal theorem. We call $f_{\wbf^*,\bbf^*}$ the {\em surrogate} and call $f$ the {\em model}.

In the typical application we have in mind, $f$ is the result of a deterministic but complex phenomenon. It is difficult to compute and its evaluation requires intensive computations or extensive memory consumption. We consider two settings. In the first setting, we can evaluate $f(x)$ for any $x\in\calD$. In the second setting, we only know a data set $(x_i,f(x_i))_{i=1..n}$ of size $n\in\NN$.

The constructed surrogate $f_{\wbf^*,\bbf^*}$ permits to rapidly evaluate an approximation of $f$ at any point of $\calD$. Of course, we want the surrogate $f_{\wbf^*,\bbf^*}$ to approximate $f$ well; but most importantly we want $f_{\wbf^*,\bbf^*}$ to overestimate $f$. In the typical scenario we have in mind, $f$ models the consumption of some resource for a critical action. Underestimating the model may cause a (potentially deadly) hazard that will not be acceptable for certification authorities especially in aeronautics (EASA or FAA). The applications include for instance: - the estimation of the braking distance of an autonomous vehicle; - the number of kilometers that can be traveled; - the maximum load a structure can carry, - the network traveling time etc Guaranteed overestimation can also be used to guarantee the fairness of a score, when underestimation on a sub-population leads to an unfair surrogate function. 
Providing such fairness or performance guarantee can be seen as a niche activity, but it alleviates the limitation that restrains the massive industrialization of neural networks in critical applications where passengers lives are at stake.

Finally, the adaptation of the method to construct an {\em under-estimation} of $f$ is straightforward and, for ease of presentation, not exposed in the paper. Of course, the combination of both an over-estimation and an under-estimation permits, for any $x$, the construction of an {\em interval} in which $f(x)$ is guaranteed to be.

\subsection{Related works}

The most standard type of guarantees in machine learning aim at upper-bounding the {\em risk} by upper-bounding the generalization error (see \cite{ab-nnltf-99} for an overview). We would like to highlight that the risk measures an average cost that does not exclude failures. Moreover, at the writing of this paper, the bounds on the generalization error are very pessimistic when compared to the performance observed in practice \cite{harvey2017nearly,bartlett2019nearly} and the recent overview \cite{jakubovitz2019generalization}. In particular, neural network are commonly learned with much less examples than required by the theory.

Other works provide guarantees on the neural network performance \cite{pmlr-v80-mirman18b,boopathy2019cnn,katz2017reluplex} that can be used to exclude failures (in our setting, the under-estimation of $f$).
The philosophy of these works is to analyze the output of the learning phase in order to provide guarantees of robustness.
Another line of research, is to impose constraints or optimize robustness criteria during the learning phase \cite{fischer2019,raghunathan2018certified}. This does not permit to guarantee an over-estimation property.

Finally, another direction makes probabilistic assumptions and provides confidence scores \cite{Gal2016Bayesian,pearce2018high,tagasovska2019,keren2018calibrated}. The confidence score does not necessarily provide a formal guarantee.

Compared to these approaches, the guarantee on the surrogate $f_{\wbf^*,\bbf^*}$ is due to the attention given to the construction of the learning problem and the hypotheses on the function $f$. To the best of our knowledge, and not considering trivial overestimation with poor performances, we are not aware of any other construction of a surrogate $f_{\wbf^*,\bbf^*}$ that is guaranteed to overestimate the model $f$.

\subsection{Hypotheses and restrictions}

In order to extend the overestimation property to the whole domain we can obviously not consider any arbitrary function $f$. In this paper, we restrict our analysis to real-valued and finite {\em non-decreasing functions}. The extension to vector valued function is straightforward. The extension to non-increasing functions is straightforward too. We deliberately  use these restrictive hypotheses to simplify the presentation. Furthermore, many extensions of the principles of the method described in the paper to other classes of non-monotone functions are possible.

Formally, for a compact domain $\calD\subset \RR^d$, a function $g:\calD \longrightarrow \RR$ is {\em non-decreasing} if and only if
\[g(x) \leq g(x') \qquad \mbox{for all } x\leq x'
\]
where the partial order on $\RR^d$ is defined for any $x=(x_k)_{k=1..d}$ and  $x'=(x'_k)_{k=1..d}\in\RR^d$ by
\[ x\leq x'  \mbox{ if and only if }  x_k \leq x'_k \mbox{, for all } k=1..d.
\]
 
 Other authors have already studied the  approximation of non-decreasing functions with neural networks \cite{lang2005monotonic,daniels2010monotone,sill_1998}. In our context, the monotonic hypothesis is motivated by the fact that monotone function are ubiquitous in industrial applications and in physics. Moreover, for a given application, where a function $f':\RR^{d'} \longrightarrow \RR$ needs to be approximated. It is sometimes possible to extract features with a function $g:\RR^{d'} \longrightarrow \RR^d$ such that the function $f:\RR^{d} \longrightarrow \RR$, satisfying $f' =f\circ g$, is monotone.

Another restriction is that the method cannot be applied when the dimension $d$ is large and $f$ varies in all directions. In fact, one contribution of the paper is to design algorithms to alleviate this problem and apply the method for larger $d$, but $d$ must anyway remain moderate for most function $f$.  

These hypotheses permit to have global guarantee that hold on the whole domain $\calD$ while weaker hypotheses on $f$ only lead to local guarantees, holding in the close vicinity of the learning examples.

\subsection{Organization of the paper}
In the next section, we define \dominate~and describe a grid based and an adaptive algorithm to construct them. The algorithms are adapted when the function $f$ can be evaluated at any new input as well as when we only know a dataset $(x_i,f(x_i))_{i=1..n}$. In Section \ref{nn-sec}, we describe a strategy to construct monotonic, over-estimating neural networks. Finally, in Section \ref{expe-sec}, we provide experiments showing that the \dominate~are located in the regions where $f$ varies strongly or is discontinuous. We also illustrate on synthetic examples that the method can accurately approximate $f$, while being guaranteed to overestimate it. An example on real data illustrates that the use of over-estimating neural networks permits to reduce the memory required to compute an over-estimation and thus permits to embed it, in a real world application.

\section{\dominate~}
\subsection{Introduction}

\begin{definition}{\bf \dominate~}

  Let $(a_i,b_i)_{i=1..m}\in \left(\RR^d\times \RR\right)^m$. Let  $\calD\subset \RR^d$ be compact and let $f:\RR^d\longrightarrow \RR$ be finite and non-decreasing. We say that  $(a_i,b_i)_{i=1..m}$ are {\em \dominate~for $f$} if and only if for any non-decreasing $g:\RR^d\longrightarrow \RR$ :
  \[\mbox{If } g(a_i)\geq b_i, \forall i=1..m, ~~ \mbox{ then } ~~ g(x)\geq f(x), \forall x\in\calD.
  \]
  \end{definition}

When $f$ is upper-bounded on $\calD$, the existence of such majoring points is established by considering: $m=1$, 
\begin{itemize}
\item $a_1$ such that $a_1 \leq x$ for all $x\in\calD$,
\item $b_1 = \sup_{x\in\calD} f(x)$.
\end{itemize}
However, for most function $f$, any non-decreasing $g$ such that $g(a_1)\geq b_1$ is a poor approximation of $f$. The goal, when constructing \dominate~of $f$ is to have $b_i\sim f(a_i)$, while $b_i\geq f(a_i)$, for all $i=1..m$. The number of points $m$ should remain sufficiently small to make the optimization of the neural network manageable.

 \subsection{Constructing \dominate~using a cover}
 For any $y$ and $y'\in\RR^d$, with $y\leq y'$, we define the {\em hyper-rectangle} between $y$ and $y'$ by
\[\calR_{y,y'} = \{x\in\RR^d | y\leq x < y' \}.
\]
Using hyper-rectangles, we define the considered covers.
\begin{definition}{\bf Cover}

Let $(y_i)_{i=1..m}$ and $(y'_i)_{i=1..m}$ in $\RR^d$ be such that $y'_i\geq y_i$, for all $i=1..m$. We  say that $(y_i,y'_i)_{i=1..m}$  {\em covers} $\calD$ if and only if
\begin{equation}\label{Dinclu}
\calD \subset \cup_{i=1}^m \calR_{y_i,y_i'}.
\end{equation}
\end{definition}
For any cover $\calC = (y_i,y'_i)_{i=1..m}$, we define the function
\begin{equation}\label{deffC}
f_\calC (x) = \min_{i ~:~ x \in\calR_{y_i,y'_i}} f(y'_i)\qquad\mbox{, for all }x\in\calD.
\end{equation}
Notice that, since $\calC$ is a cover, $\{i| x \in\calR_{y_i,y'_i}\}\neq\emptyset$ and $f_\calC(x)$ is well-defined.  We can establish that the function $f_\calC$ overestimates $f$ over $\calD$:
\begin{equation}\label{iqrvhuit}
\mbox{For all } x\in\calD, \qquad  f_\calC (x) \geq  f(x).
\end{equation}
Indeed, for any $x\in\calD$, there exists $i$ such that $x\in\calR_{y_i,y'_i}$ and $f_\calC(x) = f(y'_i)$. Therefore, since $f$ is non-decreasing and $x\leq y'_i $, we have
\[f(x) \leq f(y'_i) = f_\calC(x).
\]
\begin{prop}{\bf A cover defines \dominate~}\label{prop2}

 Let  $\calD\subset \RR^d$ be compact. Let $\calC = (y_i, y'_i)_{i=1..m}$ be a cover of $\calD$ and let $f:\RR^d\longrightarrow \RR$ be finite and non-decreasing. The family $(y_i,f(y'_i))_{i=1..m}$ are \dominate~for $f$.
\end{prop}
\begin{proof}
We consider a non-decreasing function $g$ such that, 
\begin{equation}\label{srgoit}
\mbox{for all } i=1..m, \qquad g(y_i) \geq f(y'_i).
\end{equation}
We need to prove that,
\[
\mbox{for all } x\in\calD, \qquad  g(x) \geq   f(x).
\]
To do so, we consider $x\in\calD$ and $i$ such  that $x\in\calR_{y_i,y'_i}$. Using respectively that $g$ is non-decreasing, \eqref{srgoit}, the definition of $f_\calC$ \eqref{deffC} and \eqref{iqrvhuit}, we obtain the following sequence of inequalities:
\begin{equation}\label{calcultmp}
g(x) \geq g(y_i)\geq f(y'_i) \geq f_\calC(x)\geq  f(x). 
\end{equation}
\end{proof}

The function $f_\calC$ can be computed rapidly using a look-up table but requires storing $(y_i,f(y'_i))_{i=1..m}$. This can be prohibitive in some applications.

To deal with this scenario, we describe in Section  \ref{nn-sec} a method to construct a neural network such that $f_{\wbf^*,\bbf^*}$ is non-decreasing and satisfies $f_{\wbf^*,\bbf^*}(y_i) \geq f(y'_i)$,  for all $i=1..m$. According to the proposition, such a network provides a guaranteed over-estimation of $f$, whose computation is rapid. The resources required to store $\wbf^*$ and $\bbf*$ are independent of $m$. We show in the experiments that it can be several orders of magnitude smaller. This makes it possible to embed the overestimating neural network when embedding the \dominate~is not be possible. This advantage comes at the price of loss of accuracy as $f_{\wbf^*,\bbf^*}(x) \geq f_\calC(x)$ ($f_{\wbf^*,\bbf^*}(x)$ has the role of $g$ in \eqref{calcultmp}).

\subsection{\dominate~construction algorithm}
\subsubsection{Introduction}

 In this section, we describe algorithmic strategies to adapt \dominate~to the function $f$. Throughout the section, we assume that we know $y_{min}$ and $y_{max}\in\RR^d$ such that 
 \begin{equation}\label{defyminmax} 
\calD\subset \calR_{y_{min},y_{max}}. 
\end{equation}
The goal is to build a cover such that $f_\calC$ is close to $f$ and $m$ is not too large. Ideally, the cover can be expressed as the solution of an optimization taking into account these two properties. However, the optimization would be intractable and we only describe an heuristic algorithm that construct a cover.

We construct \dominate~according to two scenarios:
\begin{itemize}
\item We can generate $f(x)$ for all $x\in\RR^d$. We call the \dominate~generated according to this setting \Fdominate.
\item We have a dataset $(x_i,f(x_i))_{i=1..n}$. It is not possible to have access to more training points. We call the \dominate~generated according to this setting \Ddominate. In order to define them, we consider the function $\tilde f:\RR^d\longrightarrow \RR$ defined, for all $x\in\calR_{y_{min},y_{max}}$, by
\begin{equation}\label{ftilde} 
 \tilde f(x) = \min_{i: x_i \geq x} f(x_i).
\end{equation}
Since $f$ is non-decreasing, we have
\[ \tilde f(x) \geq f(x) \qquad\mbox{for all }x\in\calR_{y_{min},y_{max}}.
\]
\end{itemize}
At the core of the constructions described below is the construction of a cover $\calC = (y_i,y'_i)_{i=1..m}$.

\subsubsection{\Gdominate}

We consider an accuracy parameter $\varepsilon>0$ and a norm $\|.\|$ on $\RR^d$. We define 
$$n_{max} = \lceil\frac{\|y_{max} - y_{min}\|}{\varepsilon} \rceil.$$

A  simple way to define \dominate~is to generate a cover made of equally spaced points between $y_{min}$ and $y_{max}$. Setting 
\[r = \frac{ y_{max} - y_{min}}{n_{max}} \in\RR^d
\]
and for all $i_0,\cdots, i_{d-1}\in\{1,\ldots,N-1\}$, we set $i=\sum_{k=0}^{d-1} i_k N^k$ and
\[\left\{\begin{array}{l}
y_i =   y_{min} + (i_0r_1,\ldots,i_{d-1} r_d)   \\
y'_i =   y_i +r
\end{array}\right.
\]
Notice that, the parameter $r$ defining the grid satisfies $\|r\| \leq \varepsilon$.
Given the cover, the \Fdominate~$(a_i,b_i)_{i=0..N^d-1}$ are defined, for $i=0..n_{max}^d-1$,  by
\begin{equation}\label{constrAB}
\left\{\begin{array}{l}
a_i =  y_i   \\
b_i =  f(y'_i)  
\end{array}\right.
\end{equation}
We can also construct a \Gdominate~when the function $f$ cannot be evaluated but a dataset $(x_i,f(x_i))_{i=1..n}$ is available by replacing in the above definition $f$ by $\tilde f$, see \eqref{ftilde}.

The \Gdominate~are mostly given for pedagogical reasons. The number of  values $f(x)$ that need to be evaluated and the number of \dominate~defining the objective of the optimization of the neural network are both equal to $n_{max}^d = O(\varepsilon^{-d})$. It scales poorly with $d$ and this restrains the application of the \Gdominate~to small $d$, whatever the function $f$. 

When $f$ does not vary much in some areas, the \dominate~in this area are useless. This is what motivates the adaptive algorithms developed in the next sections.

\subsubsection{Adaptive \dominate}

Bellow, we describe a {\em \DA}~ that permits to generate an adaptive cover with regard to the variations of the function $f$ (or $\tilde f$). It begins with a cover made of the single hyper-rectangle $\calR_{y_{min}, y_{max}}$. Then it decomposes every hyper-rectangle of the current cover that have not reached the desired accuracy. The decomposition is repeated until all the hyper-rectangle of the current cover have the desired accuracy (see Algorithm \ref{AlgAdapGrid}). 

Initially, we have $\calD\subset \calR_{y_{min}, y_{max}}$. For any hyper-rectangle $\calR_{y,y'}$, denoting $r= \frac{y'-y}{2}$, the decomposition replaces $\calR_{y,y'}$ by  its sub-parts as defined by 
\begin{multline}\label{decomp-eq}
\big\{\calR_{y+\left(s_1r1,\ldots,s_dr_d),y+(s_1+1)r1,\ldots,(s_d+1)r_d\right) } | \\
(s_1,\ldots,s_d)\in\{0,1\}^d\big\}.
\end{multline}
(Hence the term `dichotomy algorithm'. ) The union of the sub-parts equal the initial hyper-rectangle. Therefore, the cover remains a cover after the decomposition. The final $\calC$ is a cover of $\calD$.

We consider a norm $\|.\|$ on $\RR^d$ and real parameters $\varepsilon>0$, $\varepsilon_f>0$ and $n_p\in\NN$. We stop decomposing an hyper-rectangle if a notion of accuracy is satisfied. The notion of accuracy depends on whether we can compute $f$ or not.
\begin{itemize}
\item When we can evaluate $f(x)$: The accuracy of $\calR_{y,y'}$ is defined by the test
\begin{equation}\label{test}
f(y') - f(y) \leq \varepsilon_f \qquad \mbox{or}\qquad \|y'-y\| \leq \varepsilon
 \end{equation}
\item When we only know a dataset $(x_i,f(x_i))_{i=1..n}$: The accuracy of $\calR_{y,y'}$ is defined by the test
\begin{equation}\label{test1}
\left\{\begin{array}{l}
\tilde f(y') - \tilde f(y) \leq \varepsilon_f \qquad \mbox{or}\qquad \|y'-y\| \leq \varepsilon\\
 \mbox{or}\qquad |\{i| x_i \in \calR_{y,y'}\}| \leq n_p
 \end{array}\right.
 \end{equation}
 where $|.|$ is the cardinal of the set.
\end{itemize}
We stop decomposing if a given accuracy is reached :
\begin{itemize}
\item when $f(y') - f(y) \leq \varepsilon_f$ or $\tilde f(y') - \tilde f(y) \leq \varepsilon_f$: This happens where the function $f$  varies only slightly.
\item when $\|y'-y\|$: This happens where the function $f$ varies strongly.
\item when $|\{i| x_i \in \calR_{y,y'}\}| \leq n_p$: This happens when the number of samples in $ \calR_{y,y'}$ does not permit to improve the approximation of $f$ by $\tilde f$ in its sub-parts.
\end{itemize}

\begin{algorithm}
  \begin{algorithmic}
    \REQUIRE $\varepsilon$ : Distance below which we stop decomposing
    \REQUIRE $\varepsilon_f$ : Target upper bound for the error in $f$ 
    \REQUIRE $n_p$ : number of  examples in a decomposable hyper-rectangle
    \REQUIRE Inputs needed to compute $f$ (resp. $\tilde f$)
    \REQUIRE $y_{min},  y_{max}$ : Points satisfying \eqref{defyminmax}
    \STATE
    \STATE $\calC \leftarrow \{\calR_{y_{min},  y_{max}}\}$
    \STATE $t \leftarrow 0$

    \WHILE{$t\neq1$}
    	\STATE $t \leftarrow 1$
      \STATE $\calC' \leftarrow \emptyset$
      \FOR{$\calR_{y,y'} \in \calC$}
        \IF{ $\calR_{y,y'}$ satisfies \eqref{test} (resp. \eqref{test1}) }
        	\STATE $\calC' \leftarrow \calC' \cup \{\calR_{y,  y'}\}$
        \ELSE
        	\STATE $t \leftarrow 0$
        	\FOR{all sub-parts $\calR$ of $\calR_{y,y'}$ (see \eqref{decomp-eq})}
        	\STATE $\calC' \leftarrow \calC' \cup \{\calR\}$
        	 \ENDFOR
        \ENDIF
      \ENDFOR
      \STATE $\calC \leftarrow \calC'$
    \ENDWHILE
    \RETURN $\calC$
  \end{algorithmic}
    \caption{Adaptive cover construction  \label{AlgAdapGrid}}
\end{algorithm}


The cover is constructed according to Algorithm \ref{AlgAdapGrid}. This algorithm is guaranteed to stop after at most $n'_{max} = \lceil \log_2\left(\frac{\|y_{max} - y_{min}\|}{\varepsilon}\right) \rceil$ iteration of the `while loop'. In the worst case, every hyper-rectangle of the current cover is decomposed into $2^d$ hyper-rectangles. Therefore, the worst-case complexity of the algorithm creates 
\[ 2^{d n'_{max}} = O(\varepsilon^{-d})
\]
hyper-rectangles. The worst-case complexity bound is similar to the complexity of the \Gdominate. However, depending on $f$, the number of hyper-rectangles generated by the algorithm can be much smaller than this worst-case complexity bound. The smoother the function $f$, the less hyper-rectangles are generated.



\section{Overestimating neural networks} \label{nn-sec}

\subsection{Monotonic Neural Networks}\label{sub_sec:RW_monotone}

 In this section, we remind known result on the approximation of non-decreasing functions with neural networks having non-negative weights and non-decreasing activation functions \cite{lang2005monotonic,daniels2010monotone,sill_1998}. 

\begin{prop}{\bf Sufficient condition to get a non-decreasing network}\label{nn-prop}

For any neural network such that:
\begin{itemize}
\item its activation functions are  non-decreasing
\item its weights $\wbf$ are non-negative
\end{itemize}
the function $f_{\wbf,\bbf}$ defined by the neural network  is non-decreasing.
\end{prop}

The conditions are sufficient but not necessary. We can think of simple  non-decreasing neural network with both positive and negative weights. However, as stated in the next theorem, neural networks with non-negative weights are universal approximators of non-decreasing functions.

\begin{thm}{\bf Universality of neural networks with non-negative weights \cite{daniels2010monotone}}\label{univ-thm}

Let $\calD\subset \RR^d$ be compact. For any continuous non-decreasing function $g:\calD \rightarrow \RR$ . For any $\eta>0$, there exist a feed-forward neural network with $d$ hidden layers, a non-decreasing activation function, non-negative weights $\wbf^*$ and bias $\bbf^*$ such that
\[ |g(x) - f_{\wbf^*,\bbf^*}(x)| \leq \eta \qquad \mbox{, for all } x\in\calD,
\]
where $f_{\wbf^*,\bbf^*}$ is the function defined by the neural network.
\end{thm}
Notice that, in \cite{daniels2010monotone}, the neural network constructed in the proof of Theorem \ref{univ-thm}  involves a Heaviside activation function. The choice of the activation function is important. For instance, with a convex activation function (like ReLU), the function defined by the neural network with non-negative weights is convex \cite{amos2017input} and may approximate arbitrary poorly a well-chosen non-decreasing non-convex function.

Theorem \ref{univ-thm} guarantees that a well-chosen and optimized neural network with non-negative weights can approximate with any required accuracy the smallest non-decreasing majorant\footnote{Notice $f_\calC$ is not necessarily non-decreasing.} of $f_\calC$, as defined in \eqref{deffC}.

\subsection{Learning the neural network}\label{num-sec}

We consider a feed-forward neural network of depth $h$. The hidden layers are of width $l$. The weights are denoted by  $\wbf$ and we will restrict the search to non-negative weights: $\wbf\geq0$. The bias is denoted by $\bbf$. We consider, for a parameter $\theta>0$, the activation function
\[\sigma(t) =\tanh(\frac{t}{\theta}) \qquad\mbox{for all }t\in\RR.
\]

We consider an asymmetric  loss function in order to penalize more underestimation than overestimation
\[l_{\beta,\alpha^+, \alpha^-,p}(t)=\left\{\begin{array}{ll}
\alpha^+ (t-\beta)^2 & \mbox{ if } t\geq \beta \\
\alpha^- |t-\beta|^p & \mbox{ if } t< \beta \\
\end{array}\right.
\]
where the parameters $(\alpha^+, \alpha^-,\beta) \in \RR^3$ are non-negative and $p\in\NN$. Notice that asymmetric loss functions have already been used  to penalize either under-estimation or over-estimation \cite{yao1996asymmetric} \cite{Julian_2019}. 

Given \dominate~$(a_i,b_i)_{i=1..m}$, we define, for all $\wbf$ and $\bbf$
\[E(\wbf,\bbf) = \sum_{i=1}^m  l_{\beta,\alpha^+, \alpha,p}( f_{\wbf,\bbf}(a_i) - b_i ).
\]
The parameters of the network optimize
\begin{equation}\label{Pb}
\argmin_{\wbf\geq 0,\bbf}~~ E(\wbf,\bbf).
\end{equation}
The function $E$ is smooth but non-convex. In order to solve \eqref{Pb}, we apply a {\em projected stochastic gradient algorithm} \cite{bianchi2012convergence}. The projection on the constraint $\wbf\geq0$ is obtained by canceling its negative entries. As often with neural network, we cannot guarantee that the algorithm converges to a global minimizer.

\subsection{Guaranteeing $f_{\wbf^*,\bbf^*}(a_i)\geq b_i$}
The parameter $\beta\geq 0$ is an offset parameter. Increasing $\beta$ leads to poorer approximations. We show in the following proposition that $\beta$ can be arbitrarily small,  if the other parameters are properly chosen.

\begin{prop}{\bf Guaranteed overestimation of the samples}\label{overestimesample-prop}

Let $\beta > 0 $, $\alpha^->0$, $p>0$. If the neural network is sufficiently large and if $\theta$ is sufficiently small, then
\begin{equation}\label{tbonustpji}
f_{\wbf^*,\bbf^*}(a_i) \geq b_i\qquad\mbox{ for all }i=1..m,
\end{equation}
for any $\wbf^*$ and $\bbf^*$ solving \eqref{Pb}.
\end{prop}
\begin{proof}
Since $\alpha^-\beta^p>0$, there exists $\eta>0$ such that 
\[m \max(\alpha^-\eta^p,\alpha^+ \eta^2) \leq\alpha^-\beta^p. 
\]
Given Theorem \ref{univ-thm}, when the network is sufficiently large and $\theta$ is sufficiently small there exist  a bias $\bbf'$ and non-negative weights $\wbf'$ such that for all $i=1..m$:
\[|f_{\wbf',\bbf'}(a_i) - (b_i+\beta)| \leq \eta.
\]
Therefore,
\begin{eqnarray}
E(\wbf^*,\bbf^*) & \leq & E(\wbf',\bbf') \nonumber\\
& \leq & \sum_{i=1}^m  \max(\alpha^-\eta^p,\alpha^+ \eta^2) \nonumber \\
& \leq &  \alpha^- \beta^p. \label{iiettgourt}
\end{eqnarray}
If by contradiction we assume that there exists $i_0$ such that
\[f_{\wbf^*,\bbf^*}(a_{i_0}) < b_{i_0}
\]
then we must have
\[E(\wbf^*,\bbf^*) \geq  l_{\beta,\alpha^+, \alpha^-,p}( f_{\wbf^*,\bbf^*}(a_{i_0}) - b_{i_0} ) > \alpha^- \beta^p.
\]
This contradicts \eqref{iiettgourt} and we conclude that \eqref{tbonustpji} holds.
\end{proof}

The proposition guarantees  that for a large network, with $\theta$ small, we are sure to overestimate the target. Because Theorem \ref{univ-thm} does not provide a configuration (depth, width, activation function) that permits to approximate any function with an accuracy $\eta$, it is not possible to provide such a configuration for the parameters in Proposition \ref{overestimesample-prop}. However, given a configuration and given weights $\wbf^*$ and bias $\bbf^*$ returned by an algorithm, it is possible to test if \eqref{tbonustpji} holds. If is does not hold, it is always possible to increase the width, depth, etc and redo the optimization. Theorem \ref{univ-thm} and Proposition \ref{overestimesample-prop}, combined with properties of the landscape of large networks such as  \cite{nguyen2017loss}, guarantee that such a strategy stops after a finite number of optimization procedure.

\section{Experiments}\label{expe-sec}

In this section, we compare the results of several learning strategies on two synthetic experiments with $d=1$ and $2$ and on a real dataset from avionic industry. The synthetic experiments permit to illustrate the method; the latter real dataset show that the method permits to construct a surrogate of a critical function that can be embedded while the true critical function cannot.

The python codes that have been used to generate the experiments, as well as additional experiments, are provided with the submission and will be made available on an open source deposit.

\subsection{Methods to be compared}

 The architecture of the network is the same for all experiments and contains $4$ fully connected layers with $64$ neurons in each layer. The memory requirement to store the network is $64\times d+4\times 64^2 + 5*64=64d+16705$ floating numbers. The size of the input layer is $d$. The size of the output layer is $1$. We compare:
\begin{itemize}
\item The {\bf $\delta$-baseline}: For a parameter $\delta\geq 0$, it is a simple neural network, with an $\ell^2$ loss. It is trained:
\begin{itemize}
\item on the points $(a_i,f(a_i)+\delta)_{i=1..m}$, when $f$ can be computed;
\item on the modified dataset $(x_i,f(x_i)+\delta)_{i=1..n}$, when $f$ cannot be computed.
\end{itemize}
The $\delta$-\baseline~is in general not guaranteed to provide an overestimation. The $0$-\baseline~is expected to provide a better approximation of the true function $f$ than the other methods but it fails to always overestimating it.

If $\delta$ is such that $f(a_i)+\delta \geq b_i$, for all $i=1..m$, the $\delta$-\baseline~is guaranteed to provide an overestimation.
\item The {\bf Overestimating Neural Network (ONN)}: Is a neural network whose parameters solve \eqref{Pb} for the parameters $\beta, \alpha^+, \alpha$ and $p$ coarsely tuned for each experiment, and depending on the context, the \Gdominate~(ONN with GMP), the \Fdominate~(ONN with FMP), the \Gdominate~(ONN with DMP). 

We always take $\theta=1$. The size of the network and the values of s $\beta, \alpha^+, \alpha$ and $p$ always permit to have $f_{\wbf^*,\bbf^*}(a_i)\geq b_i$, for all $i=1..m$. Therefore, as demonstrated in the previous sections, $f_{\wbf^*,\bbf^*}$ is guaranteed to overestimate $f$.

\subsection{Evaluation metrics}

We are defining in this section the metrics that are used to compare the methods. Some metrics use a test dataset $(x'_i,f(x'_i))_{i=1..n'}$. 

For the $1D$ synthetic example, we take $n'=100000$ and for the industrial example, we take $n'=75 000$. In both cases the $x'_i$ are iid according to the uniform distribution in $\calD$. We consider the {\bf Majoring Approximation Error (MAE)} defined by
\[MAE = \frac{1}{m} \sum_{i=1}^m (b_i - f(a_i));
\]
the {\bf Root Mean Square Error (RMSE)} defined by
\[\left( \frac{1}{n'} \sum_{i=1}^{n'} ( f_{\wbf^*,\bbf^*}(x'_i)- f(x'_i))^2 \right)^{\frac{1}{2}};
\]
the {\bf Overestimation proportion (OP)} defined by
\[ \frac{100}{n'} \left|\left\{ i |  f_{\wbf^*,\bbf^*}(x'_i)\geq f(x'_i) \right\}\right|;
\]
and remind if the method  guarantees $f_{\wbf^*,\bbf^*}$ overestimates $f$ {\bf Formal Guarantee (FG)}.

For the experiments on the $1D$ synthetic example the methods are also evaluated using visual inspection.

\end{itemize}

\subsection{$1D$ synthetic experiment}

The \toyone~ aims at overestimating the function $f_1$ defined over $[-10,10]$ by

\begin{equation}\label{fonc1D}
f_1(x)=\left\{\begin{array}{ll}
 3x+3\sin(x)-4 &  \mbox{if } x \in  [-10;-1) \\
-sign(x).x^2+\sin(x) & \mbox{if } x \in [-1;1] \\
 x+\cos(x)+10 &  \mbox{if } x \in (1;10]
\end{array}\right.
\end{equation}
\begin{figure}
  \centering
      \includegraphics[width=.40\textwidth]{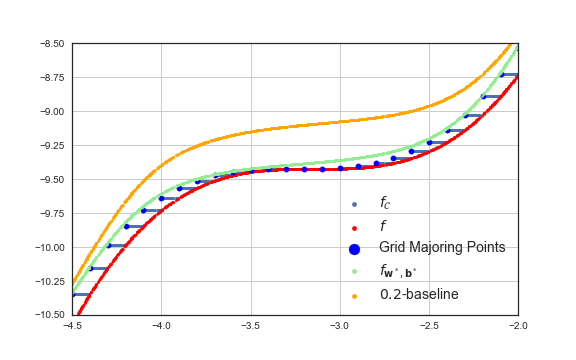}
      \caption{\toyone~ with \Gdominate \label{fig-toy1D-GDP01}}
\end{figure}

The function $f_1$, $f_\calC$, $f_{\wbf^*,\bbf^*}$ for \Gdominate~and the $0.2$-\baseline~ are displayed on Figure \ref{fig-toy1D-GDP01}, on the interval $[-4.5, -2]$. The  function $f_1f$, the \Gdominate~ and  $f_{\wbf^*,\bbf^*}$ for \Gdominate~and the $0.2$-\baseline~are displayed on Figure \ref{fig-toy1D-GDP02}, on the interval $[0.8, 1.2]$. The  function $f_1$, the \Ddominate, the sample $(x_i,f_1(x_i))_{i=1..n}$ and  $f_{\wbf^*,\bbf^*}$ for \Ddominate~ are displayed on Figure \ref{fig-toy1D-DMP01}, on the interval $[-5.5, 0]$. 
\begin{figure}
  \centering
      \includegraphics[width=.40\textwidth]{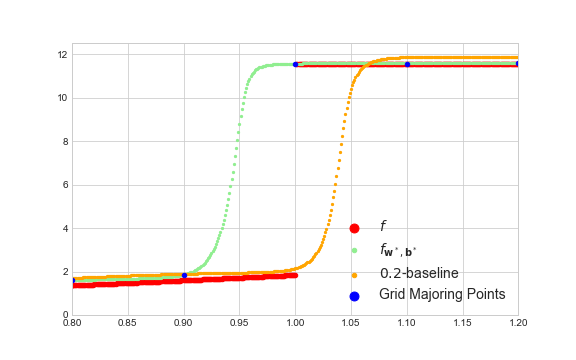}
      \caption{Discontinuity of \toyone~ with \Gdominate\label{fig-toy1D-GDP02}}
\end{figure}

\begin{figure}
  \centering
      \includegraphics[width=.40\textwidth]{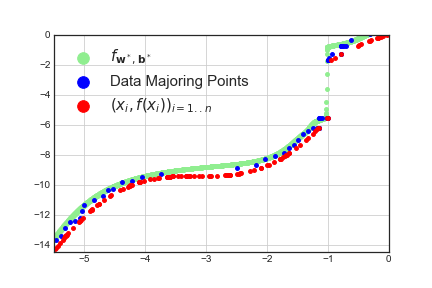}
      \caption{ \toyone~ with \Ddominate \label{fig-toy1D-DMP01}}
\end{figure}

We clearly see that the adaptive \dominate~aggregate in dyadic manner in the vicinity of the discontinuities. We also see on Figure \ref{fig-toy1D-GDP02} how \dominate~permit to anticipate the discontinuity and construct an overestimation. The density of \dominate~depends on the slope of $f_1$. This permits to reduce the approximation error. Also,  $f_{\wbf^*,\bbf^*}$ passes in the vicinity of the \dominate~ and provides a good approximation that overestimates $f$. 

The MAE, RMSE, OP and FG are provided in Table \ref{tab1D} for the $0$-\baseline, the  $0.5$-\baseline, $f_\calC$ and the ONN for three types of \dominate. We see that the RMSE worsen as we impose guarantees and approach the realistic scenario where the surrogate is easy to compute and does not require too much memory consumption.

\begin{table}
  \center\small
    \begin{tabular}{|l||c|c|c|c|c|}
      \hline
      & $n$ & \textbf{MAE} & \textbf{RMSE} & \textbf{OP} & \textbf{FG}\\
      \hline
      \hline
      $0$-\baseline & 200 & N/A& -.04 & 51.9 \% & NO \\
      $0.5$-\baseline& 200 & 0.5 & 0.59& 99.5\% & NO \\
      $f_\calC$ & 200& N/A& 0.10 & 100 \% &YES  \\
      ONN with GMP & 200 & 0.24& 0.23& 100\% & YES \\
      ONN with FMP & 200 & 0.23& 0.55& 100\% & YES \\
      ONN with DMP & 500 & 0.82 & 0.60& 100\% & YES \\
      \hline

    \end{tabular}
    \caption{Metrics - \toyone~\label{tab1D}}
\end{table}

\subsection{$2D$ synthetic experiment}
The same phenomenon described on \toyone~occur in dimension $2$. We only illustrate here the difference between the \Gdominate, \Fdominate~and \Ddominate~ for the function
\[ \forall (x,y) \in [0;15]^2\qquad g(x,y) = f_1\left(\sqrt{x^2 + y^2} - 10\right)
\]
where $f_1$ is the function defined in \eqref{fonc1D}. 

We display, on Figure \ref{staticAdaptGrid}, the \Fdominate~returned by the \DA. The density of points is correlated with the amplitude of the gradient of the function. Algorithm \ref{AlgAdapGrid} permit to diminish the number of \dominate. For instance, for the $2D$ synthetic example for $\varepsilon =0.1$, the \Gdominate~counts $22500$ points. The \Fdominate~counts $14898$ points, for $\varepsilon= 0.5$, and $3315$, for  $\varepsilon= 2$. The \Ddominate~counts $ 8734$ points, for $\varepsilon= 0.5$, and $ 1864$, for  $\varepsilon= 2$.

\begin{figure}
  \centering
  \begin{subfigure}[t]{0.23\textwidth}
    \includegraphics[width=\textwidth]{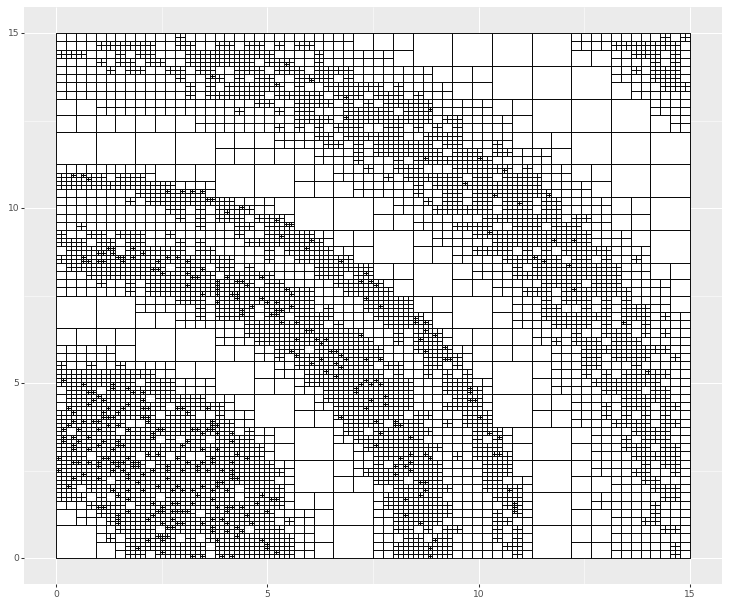}
    \caption{  Complete domain}
  \end{subfigure}
  \begin{subfigure}[t]{0.23\textwidth}
    \includegraphics[width=\textwidth]{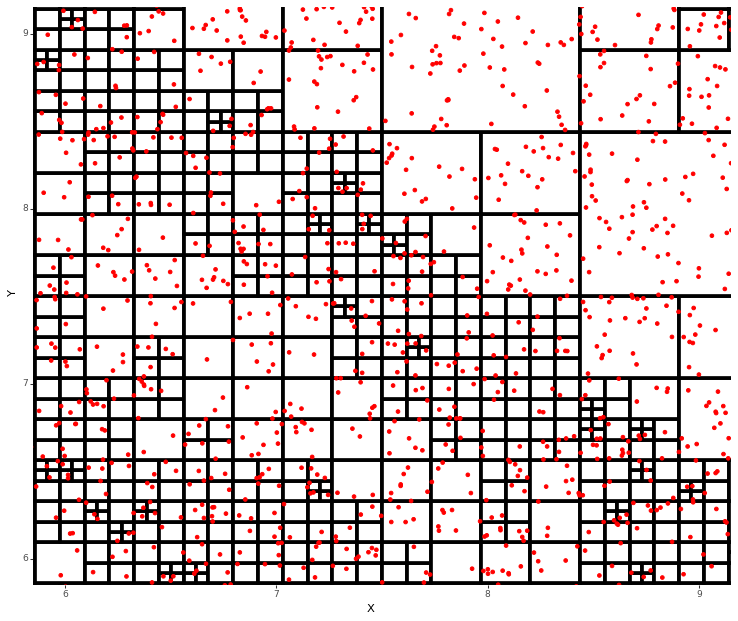}
    \caption{Zoom on the discontinuity}
  \end{subfigure}
  \caption{  \label{DataAdaptGrid} Data \dominate~grid}
\end{figure}

On Figure \ref{DataAdaptGrid}, we represent the dataset and the cover obtained using Algorithm \ref{AlgAdapGrid} for the synthetic $2D$ example. The inputs $a_i$ of the corresponding \dominate~are displayed on Figure \ref{staticAdaptGrid}.

\begin{figure}
  \centering
  \begin{subfigure}[t]{0.23\textwidth}
    \includegraphics[width=\textwidth]{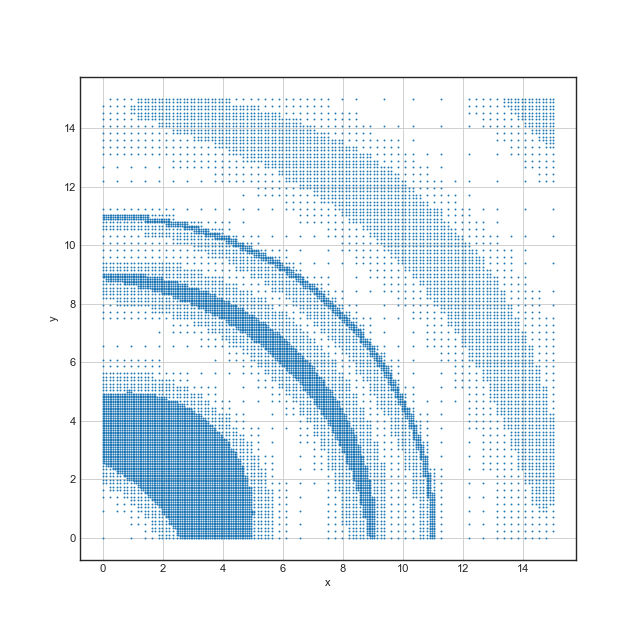}
    \caption{  Complete domain}
  \end{subfigure}
  \begin{subfigure}[t]{0.23\textwidth}
    \includegraphics[width=\textwidth]{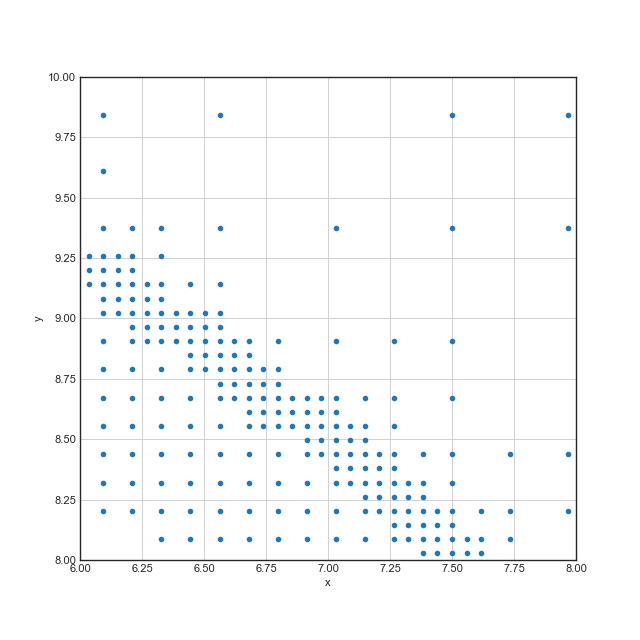}
    \caption{Zoom on the discontinuity}
  \end{subfigure}
  \caption{  \label{staticAdaptGrid} \Fdominate~on synthetic $2D$ $f$.}
\end{figure}

\subsection{Industrial application}

\begin{table}[hb]
  \center\scriptsize
  \begin{tabular}{|c||c|c|c|c|c|c|}
    \hline
     Method  & $n$ & $n_{maj}$ & \textbf{RMSE} &\textbf{MAE}& \textbf{OP}& \textbf{FG}\\
    \hline\hline
      $0$-\baseline & 150k& 150k& 3.3& N/A&65.1\%& NO\\
      $300$-\baseline & 150k& 150k& 302.6& 300.0 &100\%& NO\\
      ONN with GMP  & 150k& 150k& 309.3& 262.7 &100\%& YES \\
      ONN with DMP & 150k0& 110k & 445.7 & N/A& N/A & YES \\

      \hline
  \end{tabular}
  \caption{\label{tabIndus}Results on the industrial dataset}
\end{table}
The method developed in this paper provides formal guarantees of overestimation that are safety guarantee directly applicable in critical embedded systems.

The construction of surrogate functions is an important subject in industry \cite{lathuiliere2019comprehensive,biannic2016surrogate,jian2017learning,sudakov2019artificial}. In this work, we are considering an industrial and heavy simulation code that has six inputs $d=6$ and one output and that represents a complex physic phenomenon of an aircraft. The output is an non-decreasing function. During the flight, given flight conditions $x$ the output $f(x)$ is compared to a threshold and the result of the test launch an action. When we replace $f$ by the overestimating surrogate $f_{\wbf^*,\bbf^*}$, the airplane launches the action less often. However, the airplane only launches the action when the action is guaranteed to be safe.

The industrial dataset contains $n=150000$ examples on a static grid and another set of $150000$ sampled according to the uniform distribution on the whole domain. For each inputs, the reference computation code is used to generate the associated true output.

We compare $0$-\baseline,$300$-\baseline~with the ONN learned on \Gdominate~and \Ddominate~methods. All the methods are learned on the static grid except OON with \Ddominate. The table \ref{tabIndus} presents the metrics for the different methods. The results are acceptable for the application and memory requirement to store and embed the neural network is $17088$ floating numbers. It is one order of magnitude smaller than the size of the dataset.

\section{Conclusion - Future Work}

We presented a method that enables to formally guarantee that a prediction of a monotonic neural network will always be in an area that preserves the safety of a system. This is achieved by the construction of the network, the utilization of majoring points and the learning phase, which allows us to free ourselves from a massive testing phase that is long and costly while providing fewer guarantees.

Our work have limitations on functions that can be a safely approximate, but this is a first step toward a safe use of neural networks in critical applications. Nevertheless, this can already be used in simple safety critical systems that verify our hypotheses. Future works will look on possibility to leverage the utilization of the monotonic hypothesis. Another direction of improvement is to build smarter algorithms that require less majoring points thanks to a better adaptation to the structure of the function. In particular, this should permit to apply the method to functions whose is input space is of larger dimension, when they have the proper structure.

\section*{Acknowledgements}
This project received funding from the French "Investing for the Future – PIA3" program within the Artificial and Natural Intelligence Toulouse Institute (ANITI) under the grant agreement ANR-19-PI3A-0004. The authors gratefully acknowledge the support of the DEEL project.\footnote{\url{https://www.deel.ai/}
}

\bibliography{safelearning.bib}

\end{document}